%% file: neurips_2026.tex
\documentclass{article}

    \PassOptionsToPackage{round}{natbib}

\usepackage[utf8]{inputenc} %
\usepackage[T1]{fontenc}    %
\usepackage[linktoc=page]{hyperref}       %
\hypersetup{
    colorlinks=true,
    linkcolor=deepblue,
    filecolor=magenta,
    urlcolor=deepblue,
    citecolor=deepblue
    }
\usepackage{url}            %
\usepackage{booktabs}       %
\usepackage{amsfonts}       %
\usepackage{nicefrac}       %
\usepackage{microtype}      %
\usepackage{xcolor}         %
\usepackage{listings}
\usepackage{amsmath}
\usepackage{graphicx}
\usepackage{natbib}
\usepackage{amsthm}
\usepackage{float}
\usepackage{amssymb}
\usepackage{wrapfig}
\usepackage{fontawesome5}
\usepackage{etoc}
\usepackage[most]{tcolorbox}
\newtcolorbox{thesisbox}{
  enhanced, frame hidden,
  borderline west={4pt}{0pt}{black!75},  %
  borderline west={4pt}{0pt}{boxframe},
  colback=boxfill,
  sharp corners,
  boxsep=0pt,
  left=12pt, right=12pt, top=8pt, bottom=8pt,
}
\usepackage[dvipsnames]{xcolor}
\definecolor{deepblue}{HTML}{211bad}
\usepackage{ulem}
\usepackage[tableposition=top]{caption}
\floatstyle{plaintop}
\restylefloat{table}
\usepackage{tabularx}

\usepackage{marginnote}
\newcommand{\defterm}[1]{\emph{#1}\marginnote{#1}}     %

\usepackage{needspace}

\definecolor{boxfill}{HTML}{FDF2F8}
\definecolor{boxframe}{HTML}{B5739D}

\newtheorem{proposition}{Proposition}
\newtheorem{definition}{Definition}
\newtheorem{corollary}{Corollary}

 \usepackage{ovmi_preprint}

\title{A Common Measure of Communication \\ for Speech Brain--Computer Interfaces}

\makeatletter
\@ifpackageloaded{ovmi_preprint}{
  \author{%
  \textbf{Dulhan Jayalath} \qquad
  \textbf{Benjamin Ballyk} \qquad
  \textbf{Oiwi Parker Jones}\\[2mm]
  {\small Neural Processing Lab (PNPL\includegraphics[height=2.2ex]{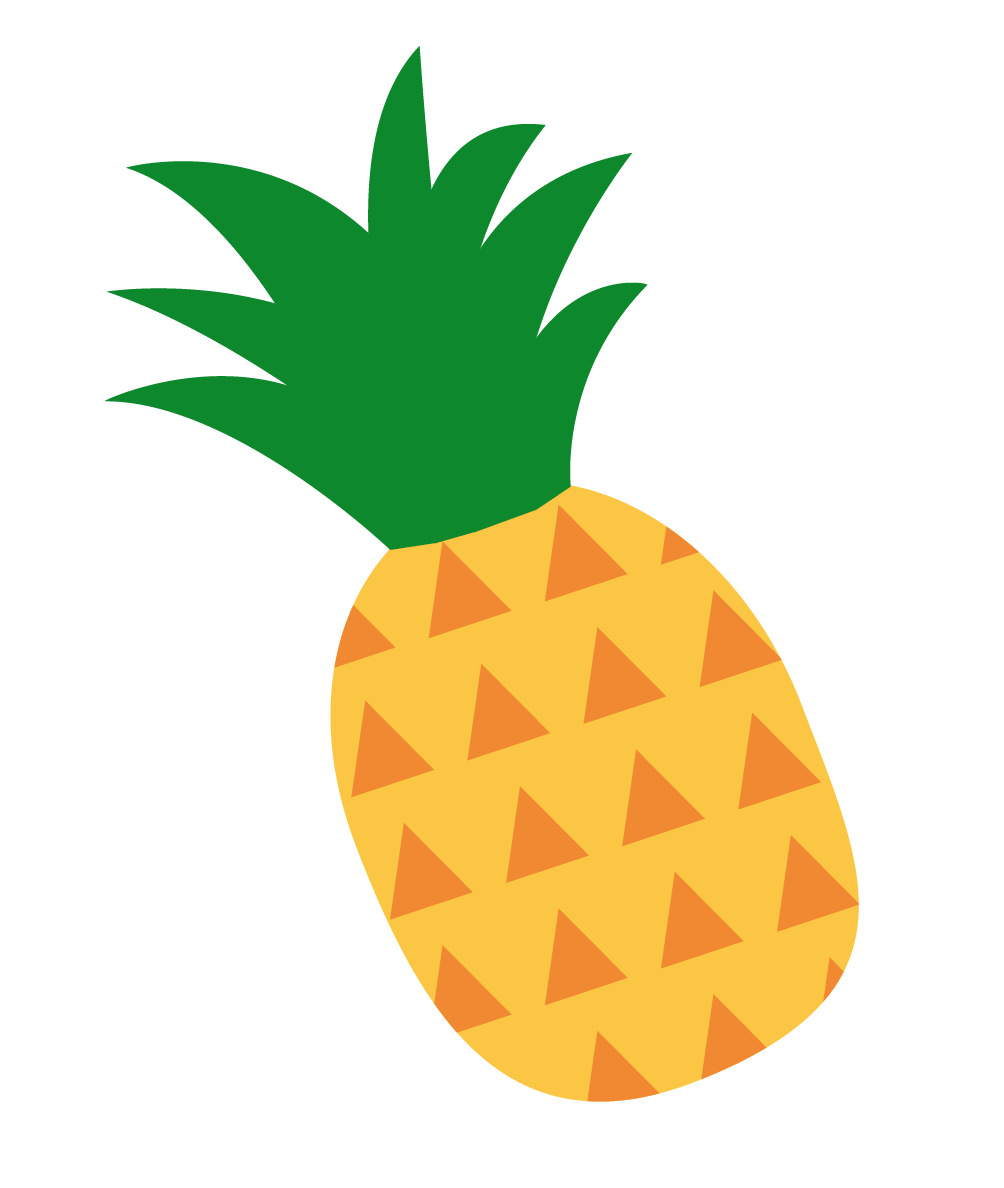}), University of Oxford}\\[3mm]
}
}{
\author{%
  Dulhan Jayalath \quad Benjamin Ballyk \quad Oiwi Parker Jones \\
  PNPL\includegraphics[height=2.2ex]{pnpl.png}, University of Oxford \\
  \texttt{\{dulhan, oiwi\}@robots.ox.ac.uk} \\
    }
}
\makeatother

\begin{document}

\maketitle

\begin{abstract}

    {

\looseness=-1
Speech brain--computer interfaces (speech BCIs) translate neural activity into language, offering a path towards restoring speech for people with paralysis and, more broadly, enabling new forms of natural human--computer interaction.
Despite this promise, the field lacks a common measure of progress because systems use different datasets, recording methods, types of speech, and vocabularies, so their reported scores are rarely comparable.
Underlying this measurement problem are two unresolved questions: (i) what distribution of words should a speech BCI enable a user to communicate, and (ii)~how much information from this distribution can a system convey. We address both by deriving \textit{open-vocabulary mutual information}~(OVMI), an information-theoretic quantity that measures the information conveyed by a decoder relative to a reference
distribution over the words a user may wish to communicate.
This allows capabilities measured under different conditions, such as distinct vocabularies, to be evaluated on a common communication scale.
We show that ordinarily reported accuracy, word error rate (WER), and other 
metrics computed only over the words a system supports
can overstate 
how much of a user's intended speech the system can communicate. We then
use OVMI to compare existing systems, expose trade-offs between how much of the user's language a system supports and how accurately it decodes those words, show that these comparisons depend on what the user is expected to communicate, and demonstrate that selecting a vocabulary to maximise OVMI yields up to 16.3\% relative improvement in accuracy across three speech domains.
OVMI therefore provides the speech BCI community with a principled way to compare heterogeneous systems, improve vocabulary design, and measure progress in the field.

\begin{center}
\vspace{1em}
\textbf{OVMI explorer} \faGlobe\, \href{https://neural-processing-lab.github.io/OVMI/}{\textcolor{black}{neural-processing-lab.github.io/OVMI/}} \\
\textbf{Python package} \faGithub\, \href{https://github.com/neural-processing-lab/OVMI}{\textcolor{black}{github.com/neural-processing-lab/OVMI}}
\end{center}

    }

\end{abstract}

\input{sections/intro_new}

\section{Information-Theoretic Evaluation in BCIs}
\label{sec:background}

\looseness=-1 The standard information metric in BCIs is the information transfer rate (ITR) popularised by \citet{Wolpaw1998EEGbasedCI, Wolpaw2002BraincomputerIF} and quoted in recent work \citep{Moses2021NeuroprosthesisFD,perkins2025sonic,neuralink2026}.
Although ITR is reported in bits per minute, Wolpaw's formulation first computes information per trial and then multiplies by the trial rate. We work at the per-trial level throughout since our primary interest is the information conveyed by each decoding attempt independently of communication speed.

\looseness=-1 With notation summarised in Appendix~\ref{app:vars}, consider a decoder with a finite vocabulary \(S\) of size \(V\). On each trial, the user intends a symbol \(X\in S\), and the decoder outputs a symbol \(Y\in S\). The quantity of interest is the mutual information \(I(X;Y)\), which measures how much observing the output \(Y\) tells us about the intended symbol \(X\). Wolpaw's derivation assumes all intended symbols are equally likely, \(p(x)=1/V\), and the decoder has symmetric errors, meaning it is correct with probability \(P\) and otherwise distributes errors uniformly over the remaining \(V-1\) symbols. Under these assumptions (derivation in Appendix~\ref{app:wolpaw}) the mutual information is
\begin{equation}
I_{\mathrm{Wolpaw}}(X;Y)
=
\log_2 V
+
P\log_2 P
+
(1-P)\log_2\!\left(\frac{1-P}{V-1}\right).
\label{eq:wolpaw-mi}
\end{equation}
The uniform prior assumption is restrictive for natural language, where words are non-uniform and Zipfian \citep{Zipf1949HumanBA}. \citet{Speier2013EvaluatingTB} addressed the problem in P300 character spellers \citep{Farwell1988TalkingOT} by replacing the uniform prior with empirical character frequencies. More recently, \citet{antonello2024many} proposed a method for information-based evaluation of continuous semantic decoding \citep{Tang2022SemanticRO}. However, these formulations still measure information within a predefined set of symbols and do not account for whether that set can represent all the symbols a user may wish to communicate.

\section{Open-Vocabulary Mutual Information}
\label{sec:ovmi-theory}

\looseness=-1
A speech BCI can convey an intended word only if the word is supported by the decoder and can be decoded reliably. 
Conventional accuracy, WER, and in-vocabulary mutual information measure only the latter, whether a decoder can reliably decode words, conditional on the intended word belonging to the decoder vocabulary \(S\). This is appropriate for a
classification task, but not for measuring communication when a user may wish
to say words outside of the supported vocabulary \(S\).

OVMI therefore evaluates the decoder relative to an external reference distribution $p$ over the words a user may wish to communicate. Its central idea is to weight information conveyed among supported words by the probability
that an intended word drawn from $p$ is supported. Thus, two decoders with identical
in-vocabulary performance can differ in communicative capability according to OVMI if their
vocabularies cover different amounts of \(p\). The toy example below illustrates a more extreme case.

\vspace{0.5\baselineskip}
\begin{center}
\begingroup
\setlength{\fboxsep}{7pt}
\fbox{%
\begin{minipage}{0.88\linewidth}
\small
\textbf{Toy example: accuracy and OVMI can rank systems differently.}

\medskip
Suppose a user may intend any of 1,000 equally likely words.

\medskip
\centering
\renewcommand{\arraystretch}{1.12}
\begin{tabular}{lcc}
\toprule
 & \textbf{System A} & \textbf{System B} \\
\midrule
Vocabulary size               & 50            & 1,000 \\
In-vocabulary accuracy        & \textbf{100\%} & 50\% \\
Lexical coverage              & 5\%           & \textbf{100\%} \\
OVMI                           & \textbf{0.28 bits} & \textbf{3.98 bits} \\
\bottomrule
\end{tabular}

\medskip
\raggedright
Despite perfect in-vocabulary decoding, System A can represent only \(5\%\) of
what the user may wish to say. System B is less accurate but supports every
intended word. Accuracy therefore ranks A above B, whereas OVMI ranks B above A
by accounting for representability and decoding fidelity.
\end{minipage}%
}
\endgroup
\end{center}
\vspace{0.5\baselineskip}

\subsection{Formal Definition}

Let \(p\) be a reference distribution over a lexicon \(\Omega\), and let \(S\subset \Omega\) denote the decoder vocabulary, with \(|S|=V\). For \(X\sim p\), define the \textit{lexical coverage}
$
C(S)=\Pr(X\in S)=\sum_{w\in S} p(w),
$
the probability that an intended word is supported by the decoder. For explicitness, we assume that if \(X\notin S\), then the user \textit{abstains} from decoding, i.e. does not attempt to think or say the word, and the observable output is \(Y=\varnothing\). If \(X\in S\), the decoder produces some \(Y\in S\). To denote whether the intended word is supported, i.e. whether \(X\in S\), we use the indicator
$
Z=\mathbf 1[X\in S]=\mathbf 1[Y\neq \varnothing].
$
\vspace{0.5\baselineskip}
\begin{proposition}[Decomposition of mutual information]
\label{prop:ovmi-decomp}
Under this model,
\begin{equation}
I(X;Y)
=
H_2(C(S))
+
I(X;Y\mid Z)
=
H_2(C(S))
+
C(S)\,I(X;Y\mid X\in S),
\label{eq:open-mi-decomp}
\end{equation}
where \(H_2(p)=-p\log_2 p-(1-p)\log_2(1-p)\) is the binary entropy. Proof in Appendix~\ref{app:ovmi-proofs}.
\end{proposition}

\begin{tcolorbox}[colback=boxfill, colframe=boxframe]
\begin{definition}[Open-vocabulary mutual information]
\label{def:ovmi}
We define
\begin{equation}
I_{\mathrm{OVMI}}(S)
:=
I(X;Y\mid Z)
=
C(S)\,I(X;Y\mid X\in S).
\label{eq:ovmi-def}
\end{equation}
\end{definition}
\textbf{Intuition.} How much lexical information is transferred per word the user would want to say?
\end{tcolorbox}

\looseness=-1 %
OVMI is a component of mutual information, stating the information conveyed among supported words, weighted by how
often a word drawn from the reference distribution is supported. The remaining
term \(H_2(C(S))\) reveals whether the intended word lies inside the decoder
vocabulary. We exclude it because it contains no information about a word's identity and counting it would assign information to detecting if a word is in the vocabulary even if the decoder is not able to decode the word.

\looseness=-1 \paragraph{Relationship to existing metrics.} When \(C(S)=1\), OVMI reduces to ordinary mutual information within the decoder
vocabulary. If the intended words are additionally uniform and decoding errors
are symmetric, it reduces to the per-trial information quantity underlying
Wolpaw's ITR
\citep{Wolpaw1998EEGbasedCI,Wolpaw2002BraincomputerIF}. For a trial rate \(r\),
\(rI_{\mathrm{OVMI}}\) therefore gives the corresponding open-vocabulary
information-transfer rate, with conventional Wolpaw ITR as a special case.
Accuracy and WER likewise measure decoding performance within the evaluated
vocabulary but do not account for the probability that an intended word is
supported. We analyse the impact of this in Section~\ref{sec:overestimate}.

\subsection{Estimating OVMI}

\paragraph{General estimator.} When a full confusion matrix of decoding errors is available, OVMI may be computed from the decoder's in-vocabulary channel, $K_S(y \mid x)=\Pr(Y=y\mid X=x)$. 
$K_S$ may be estimated by normalising the rows of the decoder's confusion matrix. OVMI weights this quantity by the reference distribution restricted to $S$. We give the full expression in Appendix~\ref{app:prop2} with Proposition~\ref{prop:ovmi-general}. 

\paragraph{Scalar estimator.} However, most existing speech BCI studies report only a scalar accuracy or WER, and large confusion matrices are often poorly estimated from limited test data.
We therefore use a Wolpaw-like approximation in which every supported word is decoded correctly with probability $P$ and errors are distributed uniformly. We define this next in Corollary~\ref{cor:ovmi-homogeneous}. We also provide a derivation without this assumption in Appendix~\ref{app:ovmi-proofs}.

\vspace{0.5\baselineskip}
\begin{corollary}[Scalar accuracy OVMI estimator]
\label{cor:ovmi-homogeneous}
Suppose that, for every intended word \(x\in S\), the decoder outputs the correct word with probability \(P\). Conditional on an error, it distributes the remaining probability uniformly across the other \(V-1\) words. Thus
\[
K_S(y\mid x)=
\begin{cases}
P, & y=x,\\[2mm]
\dfrac{1-P}{V-1}, & y\neq x.
\end{cases}
\]

Under this model, the output distribution on \(S\) is
\begin{equation}
q_S(j)
=
P\,p_S(j)
+
\frac{1-P}{V-1}\bigl(1-p_S(j)\bigr),
\qquad j\in S.
\label{eq:qS-homogeneous}
\end{equation}
Consequently (with proof in Appendix~\ref{app:ovmi-proofs}),
\begin{equation}
I_{\mathrm{OVMI}}(S)
=
C(S)\Bigl[H(q_S)-h_V(P)\Bigr].
\label{eq:ovmi-homogeneous}
\end{equation}
\end{corollary}

\paragraph{Defining $P$.} We instantiate the scalar $P$ in Eq.~\eqref{eq:ovmi-homogeneous} as macro accuracy $P_\mathrm{macro}$, the uniform average of per-word correct-decoding probabilities over $S$, rather than as a frequency-weighted (micro) average. 
The non-uniform source distribution is already modelled
by \(p\) and enters through \(C(S)\), \(p_S\), and
the output distribution \(q_S\). Micro accuracy weights classes according to the evaluation set's frequency distribution, whereas source weighting in OVMI should be determined by the external reference $p$. Macro accuracy avoids this by expressing how reliably the decoder distinguishes an average symbol in $S$, leaving frequency to enter through the source distribution.

\paragraph{Which estimator to use.} We adopt the scalar form (with pseudocode in Appendix~\ref{app:ovmi-pseudocode}) unless otherwise stated because it admits comparison with prior work that reports only accuracy and vocabulary size.
When reliable
per-word accuracy estimates are available, OVMI can be computed using the
word-specific variant in Appendix~\ref{app:ovmi-perword}; when a reliable
empirical confusion matrix is available, it should be computed directly
from Proposition~\ref{prop:ovmi-general} (statement and proof in Appendix~\ref{app:ovmi-proofs}). In the non-invasive settings considered here, however, the \(O(|S|^2)\)
parameters of a confusion matrix are often poorly estimated under the small, long-tailed held-out sets typical of
naturalistic neural recordings \citep{Hamilton2018TheRW}.

\looseness=-1

\section{Common Metrics Overestimate Open-Vocabulary Performance}
\label{sec:overestimate}

Before comparing empirical speech decoders, we first isolate the error introduced by the metric itself. Consider a noiseless decoder ($P=1$) whose vocabulary $S$ consists of the top-$V$ most frequent words under the reference distribution. With Wolpaw's formulation, such a decoder conveys $\log_2 V$ bits because the $V$ supported words are assumed to be equally likely. Accounting for their actual, non-uniform frequencies reduces this quantity to the in-vocabulary entropy $H(p_S)$, but this assumes that the words a user intends to communicate always lie within $S$. OVMI corrects this misalignment by accounting for the probability that an intended word is supported at all, giving $C(S)H(p_S)$ with an ideal, noiseless decoder. 

Figure~\ref{fig:idealdecomp}a shows that these three quantities diverge substantially.
Of the three, OVMI is the most conservative measure of information transfer, as $C(S)H(p_S) \le H(p_S) \le \log_2(V)$, with equality on the left when $S$ provides perfect coverage, $C(S)=1$, of the reference distribution, and equality on the right when $S$ is uniformly distributed.
Especially for small vocabularies, open-vocabulary metrics reveal that a system may convey only a small fraction of the information suggested by in-vocabulary entropy.

\begin{figure}
    \centering
    \includegraphics[width=1.0\linewidth, trim=0 0.4cm 0 0.25cm,clip]{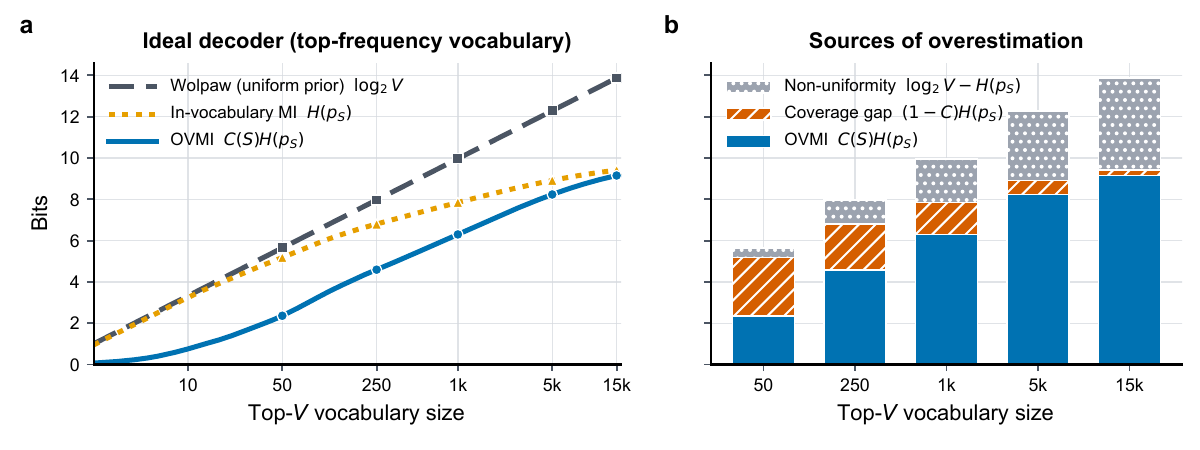}
    \caption{\textbf{In-vocabulary measures overestimate open communication.}
\textbf{(a)} Information attributed to a noiseless decoder ($P=1$) as the vocabulary size $V$ increases, with vocabularies formed from the top-$V$ most frequent words in $p$. Wolpaw (uniform prior) assigns $\log_2 V$ bits, accounting for non-uniform word frequencies gives $H(p_S)$, and OVMI further accounts for coverage, yielding $C(S)H(p_S)$. The distribution $p$ is computed from a spoken-language frequency norm derived from film and television subtitles \citep{vanHeuven2014SubtlexUKAN}, representing broad spoken English.
\textbf{(b)} Decomposition of the uniform-prior quantity,
$\log_2 V = C(S)H(p_S) + (1-C(S))H(p_S) + [\log_2 V - H(p_S)]$,
into OVMI, missing coverage, and the uniform-prior assumption.
}
    \label{fig:idealdecomp}
\end{figure}

\paragraph{What causes overestimation?} Figure~\ref{fig:idealdecomp}b separates the discrepancy into its two sources. For a noiseless decoder, the Wolpaw uniform-prior quantity may be decomposed as
\begin{equation}
\log_2 V
=
\underbrace{C(S)H(p_S)}_{\text{OVMI}}
+
\underbrace{(1-C(S))H(p_S)}_{\text{Coverage gap}}
+
\underbrace{\left[\log_2 V-H(p_S)\right]}_{\text{Non-uniformity}}.
\label{eq:overestimation_decomposition}
\end{equation}
The first term is OVMI, the second is excess due to evaluating in-vocabulary information without weighting by lexical coverage, and the third is error introduced by treating the supported words as uniformly distributed. At small $V$, in-vocabulary entropy and the uniform prior vastly overestimate information transfer, driven by incomplete lexical coverage, $C(S) \ll 1$, as even a perfectly accurate decoder cannot communicate an intended word that it does not support. 
As $V$ increases, coverage approaches unity and OVMI approaches in-vocabulary entropy, however, the gap to the uniform-prior, $\log_2(V)$, continues to grow due to the Zipfian distribution of words in the reference $p$. Invasive systems \citep{Moses2021NeuroprosthesisFD, Willett2023AHS} and recent non-invasive decoders~\citep{dAscoli2025TowardsDI, jayalath2026meg} use vocabularies of 50--250 Zipfian-distributed words, where both sources of overestimation are substantial.

\looseness=-1 \paragraph{Accuracy and WER overstate communication.} For a given vocabulary, Wolpaw's quantity is a monotonic transformation of in-vocabulary accuracy. Hence, accuracy inherits the issue of conditioning on this vocabulary. Under our model, the probability of decoding an intended
word is instead
\[
\Pr(\hat X = X)
=
\Pr(X \in S)\Pr(\hat X = X \mid X \in S)
=
C(S)P_{\mathrm{in\text{-}vocab}}.
\]
As a result, even perfect in-vocabulary accuracy can correspond to poor open-vocabulary communication when lexical coverage is low. Likewise, WER measures decoding errors on an evaluation transcript but does not by itself account for how much of an external reference distribution that system can represent. Consequently, high accuracy or low WER can occur at the same time as low open-vocabulary communication when the vocabulary does not cover much of what a user may wish to communicate.

\section{Results}
\label{sec:experiments}

\looseness=-1 We now use OVMI to evaluate existing speech BCI systems relative to an explicit
communication distribution. Unless otherwise stated, we use SUBTLEX-UK
\citep{vanHeuven2014SubtlexUKAN} as \(p\), treating its word frequencies
from film and television subtitles as a broad spoken English reference.
We first compare heterogeneous speech decoders with OVMI
(Section~\ref{sec:common-scale}), then examine how the comparison changes
with \(p\) (Section~\ref{sec:refdist}), and finally test OVMI as an objective
for vocabulary selection (Section~\ref{sec:optimise}). We provide full details on estimating OVMI in Appendix~\ref{app:estimation}.

\subsection{A Common Scale for Heterogeneous Speech BCIs}
\label{sec:common-scale}

\looseness=-1 We compare speech decoding systems that differ in vocabulary, dataset, task, and recording modality by evaluating each against the same reference distribution. For invasive work, the dataset and the decoder are typically reported together. We consider three landmark invasive studies \citep{Moses2021NeuroprosthesisFD, Willett2023AHS, Card2024AnAA} which decode attempted speech from partially paralysed patients. %
For non-invasive work, datasets are often released independently and subsequently evaluated with multiple decoding methods, so we use the strongest evaluated decoder for each dataset. For MEG-MASC \citep{gwilliams_introducing_2023}, we train MEG-XL~\citep{jayalath2026meg}, while for the other MEG datasets---LibriBrain100 \citep{mantegna2026} and Armeni \citep{armeni_10-hour_2022}---we train the decoder of \citet{dAscoli2025TowardsDI}. \citet{Tang2022SemanticRO} is included as a study in which the dataset and decoding method are reported together. In these datasets, participants listen to or view naturalistic stimuli such as movies or spoken narratives.

\begin{figure}[t]
    \centering
    \includegraphics[width=1.0\linewidth,trim=0 0.4cm 0 0.15cm,clip]{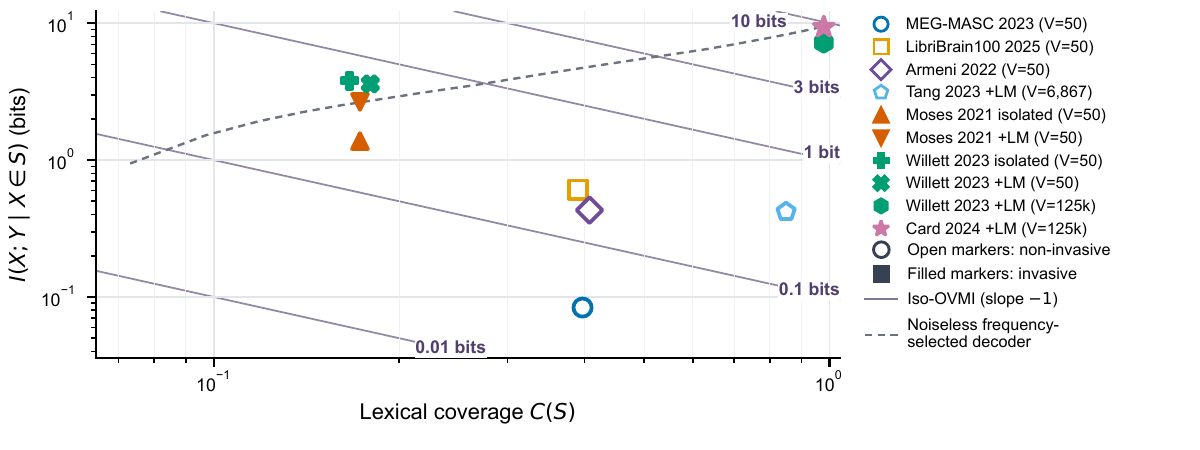}
    \caption{\textbf{Speech BCIs lose information through lexical coverage or decoding fidelity.} The product of the horizontal and vertical axes gives OVMI, with diagonal contours indicating equal OVMI. The dashed curve shows a noiseless decoder whose vocabulary consists of the top-$V$ most frequent words under the reference distribution. For frequency-selected vocabularies, vertical displacement from the noiseless curve reflects decoding error while movement leftwards reflects information foregone due to the limited coverage of smaller vocabularies.
}
    \label{fig:decomp}
\end{figure}

\looseness=-1 \paragraph{Where is information lost?} 
Figure~\ref{fig:decomp} separates the two factors determining OVMI. The
50-word invasive systems~\citep{Moses2021NeuroprosthesisFD,Willett2023AHS}
convey comparatively high information within their supported vocabularies but
cover only a small fraction of broad spoken English. Most evaluated
non-invasive systems achieve greater coverage through frequency-selected
vocabularies, but substantially lower in-vocabulary information. The later
125k-word invasive systems~\citep{Willett2023AHS,Card2024AnAA} occupy the
upper-right of the plot, combining near-complete coverage with high decoding
fidelity. For frequency-selected vocabularies, displacement below the noiseless frontier
reflects information lost through decoding error, while movement leftwards
reflects limited lexical coverage. Curated vocabularies do not need to lie below this
frontier because they may trade coverage for greater in-vocabulary entropy.
The large-vocabulary invasive systems improve primarily by expanding
coverage after decoding fidelity had already become strong, whereas the evaluated
non-invasive systems remain limited by decoding fidelity, likely as a consequence of the lower signal fidelity afforded by sensors outside the skull.

\begin{figure}[t]
    \centering
    \includegraphics[trim=0cm 0.75cm 0cm 0.5cm, clip,width=1.0\linewidth]{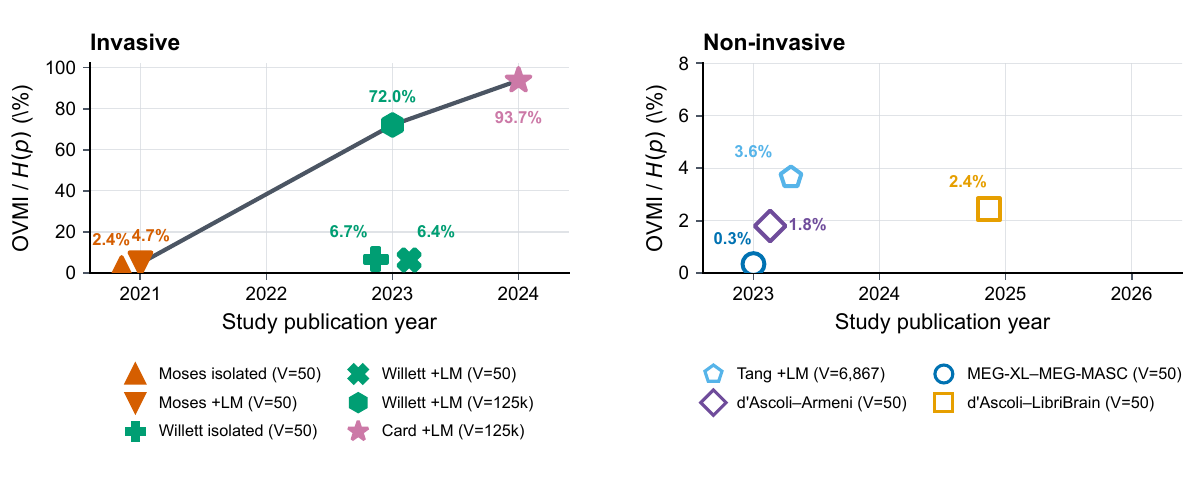}
    \caption{\textbf{Speech-BCI progress measured on a common open-vocabulary scale.}
Reported speech-decoding systems are plotted by publication year using OVMI normalised by the entropy of the reference distribution.
We include prior systems which state a vocabulary $S$ and an accuracy or WER sufficient to estimate OVMI.
The solid line connects the highest-scoring evaluated invasive system in each publication year.
Non-invasive results are unconnected because they correspond to different datasets rather than successive methods evaluated under a common experimental setting. Values provide a retrospective comparison on a common communication scale. An extended tabular version of these results with uncertainties is available in Appendix~\ref{app:table}.}
    \label{fig:progress}
\end{figure}

\begin{figure}[t]
    \centering
    \includegraphics[trim=0cm 0.25cm 0cm 0.75cm, clip,width=1.0\linewidth]{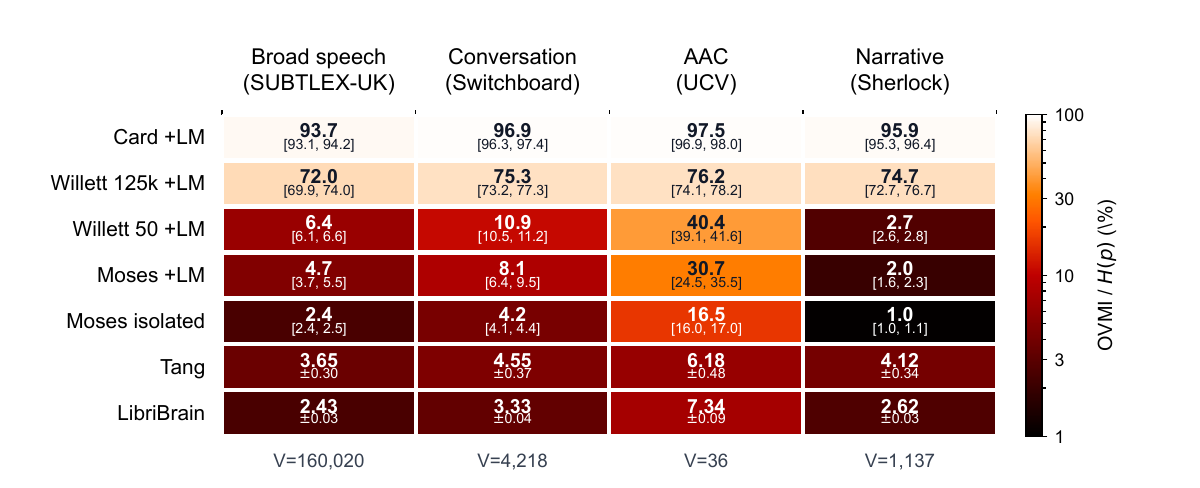}
    \caption{\textbf{Speech BCI performance depends on the intended communication domain.}
Each cell reports OVMI normalised by the entropy of the corresponding reference distribution, as a percentage, with uncertainty beneath. Brackets denote propagated intervals for published invasive results, whereas $\pm$ denotes one-SEM uncertainty for Tang and LibriBrain100. Armeni and MEG-MASC are excluded for brevity. Systems are evaluated against four communication references spanning broad spoken English \citep[SUBTLEX-UK;][]{vanHeuven2014SubtlexUKAN}, conversational speech \citep[Switchboard;][]{Godfrey1992SWITCHBOARDTS}, augmentative and alternative communication \citep[AAC/UCV;][]{erickson2019universal}, and narrative speech \citep[Sherlock Holmes;][]{doyle1892sherlock}. 
}
    \label{fig:heatmap}
\end{figure}

\looseness=-1 \paragraph{Understanding the progress of systems over time.} Figure~\ref{fig:progress} shows that in 2021, the 50-word Moses system conveys $2.4\%$ of the lexical information available in its isolated-word setting, increasing to $4.7\%$ with language-model post-processing. The 50-word Willett results in 2023 remain in the same regime, conveying $6.7\%$ and $6.4\%$ respectively. The major change comes from increasing the vocabulary with Willett's 125k-word system reaching $72.0\%$, followed by Card at $93.7\%$ in 2024. Thus, the largest historical increase among the invasive systems occurs with the transition to large-vocabulary decoding, after
in-vocabulary decoding fidelity had already approached saturation at smaller
vocabulary sizes.

\looseness=-1 Non-invasive word decoding remains in a lower information regime on this scale.
The \citet{dAscoli2025TowardsDI} decoder yields \(2.4\%\) on LibriBrain100 and \(1.8\%\) on
Armeni, MEG-XL yields \(0.3\%\) on MEG-MASC, and Tang's fMRI system reaches
\(3.6\%\). The particularly low MEG-MASC value occurs in a challenging regime
with little data per participant across a large and heterogeneous subject
population. For perceived speech, OVMI should be interpreted as
lexical decoding capability relative to the same speech distribution,
rather than communication. When comparing across settings, it is important to note the underlying paradigm; for example, perceived speech decoding does not by itself imply a useful communication interface.

\subsection{Information Transfer Depends on the Communication Distribution}
\label{sec:refdist}

\looseness=-1 OVMI is meaningful only relative to a specified communication target.
Figure~\ref{fig:heatmap} evaluates the same systems against four reference
distributions representing broad spoken English (SUBTLEX-UK), conversation
(Switchboard), augmentative and alternative communication (AAC; UCV), and
narrative speech (Sherlock). Because these distributions have different
entropies, we normalise OVMI by \(H(p)\).

\looseness=-1 If information transfer were a property of the decoder alone, then there would be no difference between the columns. Instead, we find that large-vocabulary systems are comparatively insensitive to the choice of
\(p\) as Card conveys \(93.7\)--\(97.5\%\) across the four references and the
125k-word Willett system \(72.0\)--\(76.2\%\). However, systems with smaller vocabularies
are much more sensitive. The 50-word Willett system, for example, conveys
\(6.4\%\) of broad spoken-English information but \(40.4\%\) under the AAC
reference and Moses similarly increases from \(4.7\%\) to \(30.7\%\). These systems use vocabularies designed around clinical communication, and consequently cover a much larger proportion of the more restricted AAC distribution than of unrestricted speech. LibriBrain100
and the isolated-word Moses system are nearly matched under broad speech
(\(2.43\%\) versus \(2.4\%\)), but Moses leads substantially under AAC
(\(16.5\%\) versus \(7.34\%\)), whereas LibriBrain100 leads under narrative
speech (\(2.62\%\) versus \(1.0\%\)). Whether one system represents progress
over another can therefore depend on what the interface is intended to
communicate. Accordingly, OVMI should always be reported together with \(p\) and when the
intended application is uncertain, several plausible reference distributions
can be reported as we do here.

\subsection{OVMI for Vocabulary Selection}
\label{sec:optimise}

\looseness=-1 Up to this point, we have used OVMI retrospectively to evaluate prior studies. We next ask whether it can also help choose the vocabulary of a system.
Recent contrastive decoders restrict a larger
candidate lexicon to a smaller vocabulary at inference to improve accuracy \citep{dAscoli2025TowardsDI,jayalath2026meg}, creating a trade-off between
supporting frequent or important words and retaining words that are easier to decode.
OVMI provides a natural objective because it accounts for both of these factors.

\begin{figure}[t]
    \centering
    \includegraphics[width=1.0\linewidth,trim=0 0.25cm 0 0.25cm,clip]{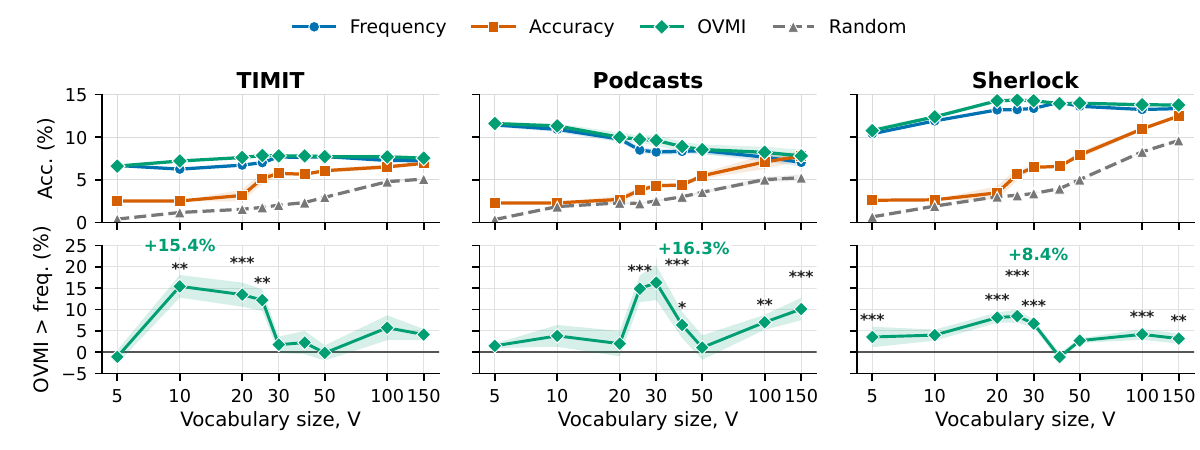}
    \caption{\looseness=-1 \textbf{OVMI-guided vocabulary selection improves accuracy.} With a decoder trained on LibriBrain100, we select vocabularies for three of its different test domains (TIMIT/Podcasts/Sherlock) according to (i) frequency in the domain training data, (ii) highest top-1 accuracy on the pooled validation data, (iii) maximising OVMI on the validation data with domain training data as $p$, and (iv) randomly. \textbf{(Top)} Held-out word accuracy, counting words outside the vocabulary as incorrect. Shading denotes the standard error across five seeds. At $V=250$, all methods converge because the candidate pool contains 250 words so we omit this trivial endpoint. \textbf{(Bottom)} Mean relative improvement of OVMI over frequency. Stars show statistical significance according to one-sided paired permutation tests ($^{*}p<.05$, $^{**}p<.01$, $^{***}p<.001$). More details in Appendix~\ref{app:optimise}.}
    \label{fig:optimise}
\end{figure}

\looseness=-1 We test OVMI as a selection objective using the \citet{dAscoli2025TowardsDI} contrastive decoder. The model is trained once with a candidate lexicon of 250 words. At inference, decoding requires retrieving the most likely word from a set of candidate embeddings. For each vocabulary size $V$, we select a subset $S_V$ and restrict retrieval to that subset, where varying $V$ does not require any changes to the trained decoder. Figure~\ref{fig:optimise} shows the results of selecting vocabularies of varying size according to frequency, validation accuracy, or OVMI. We repeat this for three different communication distributions: TIMIT, which has sentences designed to cover a diverse range of phonetic contexts, a set of podcast conversations, and a Sherlock Holmes book chapter. We measure the resulting test accuracy and find that OVMI matches or exceeds all alternatives across vocabulary sizes. It improves results over frequency selection, the strategy used in recent non-invasive decoders \citep{dAscoli2025TowardsDI, jayalath2025unlocking, jayalath2026meg}, at smaller vocabulary sizes, with peak relative improvements of 15.4\%, 16.3\%, and 8.4\%. The advantage diminishes at larger vocabulary sizes as expected when the choice of words becomes less restricted.

\section{Discussion}
\label{sec:discussion}

\looseness=-1 \paragraph{Scope and limitations.} Three caveats bound this work. 
(i) Our retrospective comparisons rely on the scalar estimator because confusion matrices are rarely reported. When richer statistics are available, OVMI can be computed directly from the empirical confusion matrix (Proposition~\ref{prop:ovmi-general}) or from per-word accuracies~(Appendix~\ref{app:ovmi-perword}).
(ii) OVMI penalises out-of-vocabulary words as unsupported, however, the meaning could be expressed through paraphrasing. We choose to keep the measure tied to the lexical information explicitly available through the decoder vocabulary.
(iii) We evaluate lexical rather than contextual information; extending OVMI to a conditional language distribution is left to future work.
{
\medskip
\begin{center}
\small
\renewcommand{\arraystretch}{1.15}
\begin{tabularx}{0.88\linewidth}{@{}X p{0.24\linewidth} p{0.12\linewidth}@{}}
\toprule
\textbf{Question} & \textbf{Evaluation type} & \textbf{Cost} \\
\midrule
Which method performs best in this setting?
    & Benchmark
    & Medium \\

What does that capability mean under $p$?
    & OVMI
    & Low \\

Is the system practically useful?
    & User-centred
    & High \\
\bottomrule
\end{tabularx}
\end{center}
\medskip

\looseness=-1
By evaluating capabilities measured in different experimental settings against the same reference distribution, OVMI places heterogeneous speech BCI systems on a common communication scale. It complements controlled benchmarks as benchmarks compare methods within a fixed setting, whereas OVMI relates capabilities across settings to the same communication objective. OVMI can also guide vocabulary design when the supported words are controllable, as our vocabulary selection experiments demonstrate. We note that the comparisons in this work do not erase differences in experimental setting, and a higher OVMI does not necessarily imply a more practical or clinically useful paradigm; for example, perceived speech decoding on its own is unlikely to constitute a usable communication interface. We recommend reporting OVMI alongside standard decoder metrics for future ease of comparison, and encourage user-centred evaluation before any practical deployment. We hope that OVMI's lasting role is to give the emerging field of speech BCIs a principled way to compare and improve different systems.

}

\clearpage
\begin{ack}
\looseness=-1 DJ would like to thank Charles London for his assistance with checking proofs and derivations, as well as Gilad Landau, Tasha Kim, Miran Özdogan, and Mélanie Schneider for reviewing early drafts of this work.

\looseness=-1 We would like to acknowledge the use of the University of Oxford Advanced Research Computing~(ARC) facility in carrying out this work. \url{http://dx.doi.org/10.5281/zenodo.22558}. We are especially grateful to the ARC support team for their timely support as conference deadlines approached.

\looseness=-1 We also thank \href{https://modal.com}{Modal Labs, Inc.} for a generous compute grant which helped support this project.

DJ is supported by an AWS Studentship from the EPSRC Centre for Doctoral Training in Autonomous Intelligent Machines and Systems (AIMS) (EP/S024050/1). OPJ and the PNPL group are supported by the MRC (MR/X00757X/1), Royal Society (RG$\backslash$R1$\backslash$241267), NSF (2314493), NFRF (NFRFT-2022-00241), SSHRC (895-2023-1022), and ARIA (SCNI-SE01-P004).
\end{ack}

\bibliographystyle{plainnat} %
\bibliography{references}

\clearpage
\appendix
\addcontentsline{toc}{part}{Appendix}
\etocsetnexttocdepth{subsection}
\etocsettocstyle{\section*{Appendix Contents}}{}
\localtableofcontents
\newpage

\section{Table of Variables}
\label{app:vars}

Table~\ref{tab:vars} summarises the core variables introduced in
Section~\ref{sec:background}~and~\ref{sec:ovmi-theory}.

\begin{table}[h]
\centering
\renewcommand{\arraystretch}{1.2}
\setlength{\tabcolsep}{6pt}
\begin{tabular}{@{}c p{0.58\linewidth} c@{}}
\toprule
\textbf{Symbol} & \textbf{Description} & \textbf{Introduced in} \\
\midrule
$\Omega$    & Reference lexicon.
            & Section \ref{sec:ovmi-theory} \\
$p$         & Reference distribution over $\Omega$.
            & Section \ref{sec:ovmi-theory} \\
$S$         & Decoder vocabulary, $S \subset \Omega$.
            & Section \ref{sec:ovmi-theory} \\
$V$         & Decoder vocabulary size, $V = |S|$.
            & Section \ref{sec:ovmi-theory} \\
$X$         & User's intended word drawn from the reference distribution, $X \sim p$.
            & Section \ref{sec:ovmi-theory} \\
$Y$         & Decoder output, $Y \in S \cup \{\varnothing\}$; $Y = \varnothing$ on abstention.
            & Section \ref{sec:ovmi-theory} \\
$Z$         & In-vocabulary indicator,
              $Z = \mathbf{1}[X \in S] = \mathbf{1}[Y \neq \varnothing]$.
            & Section \ref{sec:ovmi-theory} \\
$C(S)$      & Lexical coverage of $S$ under $p$,
              $C(S) = \Pr(X \in S) = \sum_{w \in S} p(w)$.
            & Section \ref{sec:ovmi-theory} \\
$p_S$       & In-vocabulary target distribution,
              $p_S(x) = p(x)/C(S)$ for $x \in S$.
            & Proposition~\ref{prop:ovmi-general} \\
$K_S$       & In-vocabulary channel,
              $K_S(y \mid x) = \Pr(Y = y \mid X = x)$ for $x, y \in S$.
            & Proposition~\ref{prop:ovmi-general} \\
$q_S$       & In-vocabulary output distribution,
              $q_S(y) = \sum_{x \in S} p_S(x)\, K_S(y \mid x)$.
            & Proposition~\ref{prop:ovmi-general} \\
$P$         & Per-symbol correct-decoding probability under the homogeneous
              symmetric channel.
            & Corollary~\ref{cor:ovmi-homogeneous} \\
$h_V(u)$    & Entropy of a $V$-ary symmetric channel row with correct
              probability $u$.
            & Equation~\eqref{eq:hV} \\
$I_{\mathrm{OVMI}}(S)$
            & Open-vocabulary mutual information,
              $I_{\mathrm{OVMI}}(S) = C(S)\, I(X; Y \mid X \in S)$.
            & Def.~\ref{def:ovmi},\, Eq.~\eqref{eq:ovmi-def} \\
\bottomrule
\end{tabular}
\caption{Core variables in the open-vocabulary mutual information
formulation. Section, proposition, corollary, and equation references
point to where each symbol is first introduced.}
\label{tab:vars}
\end{table}

\section{Derivation of Wolpaw's Information Transfer Rate}
\label{app:wolpaw}

Under Wolpaw's standard derivation, two simplifying assumptions are made. First, all intended symbols are equally likely, so \(p(x)=1/V\). Second, the decoder is assumed to have symmetric errors: it is correct with probability \(P\), and when it is wrong, each of the other \(V-1\) symbols is equally likely. Thus
\[
\Pr(Y=y\mid X=x)=
\begin{cases}
P, & y=x,\\[2mm]
\dfrac{1-P}{V-1}, & y\neq x.
\end{cases}
\]

These assumptions make the calculation straightforward. Fix any output symbol \(y\). There is exactly one intended symbol, namely \(x=y\), that produces \(y\) as the correct output, contributing probability mass \(P\). The remaining \(V-1\) intended symbols can also produce the same \(y\), but only as an error, and each contributes \((1-P)/(V-1)\). Averaging over the uniform prior therefore gives
\[
p(y)
=
\sum_{x\in S} p(y\mid x)p(x)
=
\frac{1}{V}\left[P + (V-1)\frac{1-P}{V-1}\right]
=
\frac{1}{V}.
\]
So the decoder output \(Y\) is itself uniform over the vocabulary, and hence
\[
H(Y)=\log_2 V.
\]

Next consider the conditional entropy. Once we fix an intended symbol \(x\), the conditional distribution of \(Y\) has one outcome with probability \(P\) (the correct symbol) and \(V-1\) outcomes with probability \((1-P)/(V-1)\) (the incorrect symbols). Therefore
\[
H(Y\mid X=x)
=
-P\log_2 P
-(V-1)\frac{1-P}{V-1}\log_2\!\left(\frac{1-P}{V-1}\right)
=
-P\log_2 P
-(1-P)\log_2\!\left(\frac{1-P}{V-1}\right).
\]
Because this expression does not depend on \(x\), averaging over \(x\) leaves it unchanged:
\[
H(Y\mid X)
=
-P\log_2 P
-(1-P)\log_2\!\left(\frac{1-P}{V-1}\right).
\]

Substituting these two terms into \(I(X;Y)=H(Y)-H(Y\mid X)\) yields the Wolpaw expression for mutual information in an in-vocabulary decoder:
\begin{equation}
I_{\mathrm{Wolpaw}}(X;Y)
=
\log_2 V
+
P\log_2 P
+
(1-P)\log_2\!\left(\frac{1-P}{V-1}\right).
\end{equation}

\section{Proofs and Supplementary Derivations for OVMI}
\label{app:ovmi-proofs}

We use the convention \(0\log_2 0 = 0\).

\subsection{Proof of Proposition~\ref{prop:ovmi-decomp}}

\begin{proof}[Proof of Proposition~\ref{prop:ovmi-decomp}]
Recall that
\[
Z=\mathbf 1[X\in S]=\mathbf 1[Y\neq \varnothing].
\]
Since \(Z\) is a deterministic function of \(Y\), adjoining \(Z\) to the
output does not change the mutual information:
\[
I(X;Y)=I(X;Y,Z).
\]
Applying the chain rule for mutual information gives
\[
I(X;Y,Z)=I(X;Z)+I(X;Y\mid Z).
\]
Because \(Z\) is also a deterministic function of \(X\),
\[
I(X;Z)=H(Z)-H(Z\mid X)=H(Z).
\]
Now \(\Pr(Z=1)=C(S)\) and \(\Pr(Z=0)=1-C(S)\), so
\[
H(Z)=H_2(C(S)).
\]
It remains to expand the conditional mutual information:
\[
I(X;Y\mid Z)
=
\sum_{z\in\{0,1\}} \Pr(Z=z)\,I(X;Y\mid Z=z).
\]
Under our model, \(Z=0\) implies \(Y=\varnothing\)
deterministically, hence
\[
I(X;Y\mid Z=0)=0.
\]
Also, \(Z=1\) is exactly the event \(X\in S\), so
\[
I(X;Y\mid Z=1)=I(X;Y\mid X\in S).
\]
Therefore
\[
I(X;Y\mid Z)=C(S)\,I(X;Y\mid X\in S).
\]
Combining the pieces yields
\[
I(X;Y)=H_2(C(S))+I(X;Y\mid Z)
      =H_2(C(S))+C(S)\,I(X;Y\mid X\in S),
\]
as claimed.
\end{proof}

\subsection{Equivalent Entropy Derivation of Proposition~\ref{prop:ovmi-decomp}}

Proposition~\ref{prop:ovmi-decomp} can also be obtained by decomposing \(H(Y)\) and \(H(Y\mid X)\).

Because \(Z=\mathbf 1[Y\neq \varnothing]\) is a deterministic function of
\(Y\), the chain rule gives
\[
H(Y)=H(Z)+H(Y\mid Z).
\]
Expanding by the two values of \(Z\),
\begin{align}
H(Y)
&= H_2(C(S))
 + \Pr(Z=1)\,H(Y\mid Z=1)
 + \Pr(Z=0)\,H(Y\mid Z=0) \nonumber\\
&= H_2(C(S))
 + C(S)\,H(Y\mid Z=1),
\label{eq:appendix-HY}
\end{align}
since \(Z=0\) implies \(Y=\varnothing\) deterministically and hence
\(H(Y\mid Z=0)=0\). Because \(Z=1\) is equivalent to \(X\in S\),
\[
H(Y\mid Z=1)=H(Y\mid X\in S).
\]
Thus
\begin{equation}
H(Y)=H_2(C(S))+C(S)\,H(Y\mid X\in S).
\label{eq:appendix-HY-final}
\end{equation}

Similarly,
\begin{align}
H(Y\mid X)
&= \Pr(X\in S)\,H(Y\mid X, X\in S)
 + \Pr(X\notin S)\,H(Y\mid X, X\notin S) \nonumber\\
&= C(S)\,H(Y\mid X, X\in S),
\label{eq:appendix-HYgivenX}
\end{align}
because \(X\notin S\) implies \(Y=\varnothing\) deterministically.

Subtracting \eqref{eq:appendix-HYgivenX} from \eqref{eq:appendix-HY-final},
\begin{align}
I(X;Y)
&= H(Y)-H(Y\mid X) \nonumber\\
&= H_2(C(S))
 + C(S)\Bigl[H(Y\mid X\in S)-H(Y\mid X, X\in S)\Bigr] \nonumber\\
&= H_2(C(S))
 + C(S)\,I(X;Y\mid X\in S),
\end{align}
recovering \eqref{eq:open-mi-decomp}.

\subsection{Proposition~\ref{prop:ovmi-general}}
\label{app:prop2}

\begin{proposition}[OVMI for a general in-vocabulary decoder]
\label{prop:ovmi-general}
Now condition on the event that the intended word is representable, i.e. \(X\in S\). The corresponding distribution over in-vocabulary target words is
$
p_S(x)=\Pr(X=x\mid X\in S)=\frac{p(x)}{C(S)},
$
where $x \in S$, that is, the unrestricted word distribution restricted to \(S\) and renormalised to sum to one.

For \(x,y\in S\), let
$
K_S(y\mid x)=\Pr(Y=y\mid X=x)
$
denote the decoder's conditional distribution over outputs when the intended word is \(x\). Thus \(K_S\) is the effective channel operating within the deployed vocabulary.

Under \(p_S\) and \(K_S\), the resulting output distribution on \(S\) is
\[
q_S(y)=\sum_{x\in S} p_S(x)\,K_S(y\mid x), \qquad y\in S.
\]
Then
\begin{equation}
I_{\mathrm{OVMI}}(S)
=
C(S)\left[
H(q_S)-\sum_{x\in S} p_S(x)\,H\!\left(K_S(\cdot\mid x)\right)
\right]
\label{eq:ovmi-general}
\end{equation}
where entropy $H(X) = -\sum_i p(x_i) \log_2 p(x_i)$ for a discrete random variable $X$.
\end{proposition}

We will use the following shorthand
\begin{equation}
h_V(u)
:=
-u\log_2 u
-(1-u)\log_2\!\left(\frac{1-u}{V-1}\right).
\label{eq:hV}
\end{equation}
This is the same as the term $H(Y|X)$ in Wolpaw's formula in Eq.~\ref{eq:wolpaw-mi}, describing the entropy of the decoder's output distribution under a symmetric error assumption where the intended word is output with probability \(u\), and each of the other \(V-1\) words is output with probability \((1-u)/(V-1)\).

\subsection{Proof of Proposition~\ref{prop:ovmi-general}}

\begin{proof}[Proof of Proposition~\ref{prop:ovmi-general}]
By Definition~\ref{def:ovmi},
\[
I_{\mathrm{OVMI}}(S)=C(S)\,I(X;Y\mid X\in S).
\]
Conditioned on \(X\in S\), the input distribution is
\[
p_S(x)=\Pr(X=x\mid X\in S)=\frac{p(x)}{C(S)},
\]
and the in-vocabulary channel is
\[
K_S(y\mid x)=\Pr(Y=y\mid X=x), \qquad x,y\in S.
\]
The output distribution is therefore
\[
q_S(y)=\Pr(Y=y\mid X\in S)=\sum_{x\in S} p_S(x)\,K_S(y\mid x).
\]
Applying the standard mutual-information identity $I(X;Y) = H(Y) - H(Y\mid X)$ under the conditional
distribution \(X\mid X\in S\),
\[
I(X;Y\mid X\in S)
=
H(q_S)-\sum_{x\in S} p_S(x)\,H\!\left(K_S(\cdot\mid x)\right).
\]
Multiplying by \(C(S)\) yields
\[
I_{\mathrm{OVMI}}(S)
=
C(S)\left[
H(q_S)-\sum_{x\in S} p_S(x)\,H\!\left(K_S(\cdot\mid x)\right)
\right],
\]
which is \eqref{eq:ovmi-general}.
\end{proof}

\subsection{Proof of Corollary~\ref{cor:ovmi-homogeneous}}

\begin{proof}[Proof of Corollary~\ref{cor:ovmi-homogeneous}]
Under the homogeneous symmetric channel,
\[
K_S(y\mid x)=
\begin{cases}
P, & y=x,\\[2mm]
\dfrac{1-P}{V-1}, & y\neq x.
\end{cases}
\]
For each \(x\in S\), the entropy of the corresponding channel row is
\[
H\!\left(K_S(\cdot\mid x)\right)
=
-P\log_2 P
-(1-P)\log_2\!\left(\frac{1-P}{V-1}\right)
=
h_V(P),
\]
where \(h_V\) is defined in \eqref{eq:hV}.

The induced output distribution satisfies
\begin{align}
q_S(j)
&= \sum_{x\in S} p_S(x)\,K_S(j\mid x) \nonumber\\
&= p_S(j)P + \sum_{x\neq j} p_S(x)\frac{1-P}{V-1} \nonumber\\
&= P\,p_S(j)+\frac{1-P}{V-1}\bigl(1-p_S(j)\bigr),
\end{align}
which is \eqref{eq:qS-homogeneous}.

Substituting into Proposition~\ref{prop:ovmi-general} gives
\begin{align}
I_{\mathrm{OVMI}}(S)
&=
C(S)\left[
H(q_S)-\sum_{x\in S} p_S(x)\,h_V(P)
\right] \nonumber\\
&=
C(S)\Bigl[H(q_S)-h_V(P)\Bigr],
\end{align}
because \(\sum_{x\in S} p_S(x)=1\). This is
\eqref{eq:ovmi-homogeneous}.
\end{proof}

\subsection{Retaining Per-Word Accuracies in OVMI}
\label{app:ovmi-perword}

\begin{corollary}[Symmetric error model with word-specific accuracies]
\label{cor:ovmi-word-dependent}
Suppose each intended word \(x\in S\) has its own correct-decoding probability \(P_c(x)\). Thus some words may be easier to decode than others. Conditional on an error, however, the decoder distributes the remaining probability uniformly across the other \(V-1\) in-vocabulary words. Hence
\[
K_S(y\mid x)=
\begin{cases}
P_c(x), & y=x,\\[2mm]
\dfrac{1-P_c(x)}{V-1}, & y\neq x.
\end{cases}
\]

Under this model, the output distribution on \(S\) is
\begin{equation}
q_S(j)
=
p_S(j)P_c(j)
+
\frac{1}{V-1}\sum_{x\neq j} p_S(x)\bigl(1-P_c(x)\bigr),
\qquad j\in S.
\label{eq:qS-word-dependent}
\end{equation}
The first term is the probability mass from correctly decoding \(j\) when \(X=j\). The second term is the total probability mass assigned to \(j\) when some other in-vocabulary word was intended but the decoder makes an error and spreads that error uniformly across the incorrect outputs. Consequently,
\begin{equation}
I_{\mathrm{OVMI}}(S)
=
C(S)\left[
H(q_S)-\sum_{x\in S} p_S(x)\,h_V\!\bigl(P_c(x)\bigr)
\right].
\label{eq:ovmi-word-dependent}
\end{equation}
\end{corollary}

This model relaxes the shared-accuracy assumption of Corollary~\ref{cor:ovmi-homogeneous}, while retaining the simplifying assumption that, conditional on an error, all incorrect outputs are equally likely.

\begin{proof}[Proof of Corollary~\ref{cor:ovmi-word-dependent}]
Under the word-dependent symmetric channel,
\[
K_S(y\mid x)=
\begin{cases}
P_c(x), & y=x,\\[2mm]
\dfrac{1-P_c(x)}{V-1}, & y\neq x.
\end{cases}
\]
For each \(x\in S\), the entropy of the \(x\)-th channel row is
\[
H\!\left(K_S(\cdot\mid x)\right)
=
-P_c(x)\log_2 P_c(x)
-(1-P_c(x))\log_2\!\left(\frac{1-P_c(x)}{V-1}\right)
=
h_V\!\bigl(P_c(x)\bigr).
\]

The induced output distribution is
\begin{align}
q_S(j)
&= \sum_{x\in S} p_S(x)\,K_S(j\mid x) \nonumber\\
&= p_S(j)P_c(j)+\sum_{x\neq j} p_S(x)\frac{1-P_c(x)}{V-1} \nonumber\\
&= p_S(j)P_c(j)+\frac{1}{V-1}\sum_{x\neq j} p_S(x)\bigl(1-P_c(x)\bigr),
\end{align}
which is \eqref{eq:qS-word-dependent}.

Substituting into Proposition~\ref{prop:ovmi-general} yields
\[
I_{\mathrm{OVMI}}(S)
=
C(S)\left[
H(q_S)-\sum_{x\in S} p_S(x)\,h_V\!\bigl(P_c(x)\bigr)
\right],
\]
which is \eqref{eq:ovmi-word-dependent}.
\end{proof}

\subsection{Further Observations for the Homogeneous Symmetric Channel}
\label{app:homogeneous-channel-observations}

A useful rearrangement of \eqref{eq:qS-homogeneous} is
\begin{equation}
q_S(j)
=
\frac{1-P}{V-1}
+
\frac{PV-1}{V-1}\,p_S(j).
\label{eq:qS-affine}
\end{equation}
Equation~\eqref{eq:qS-affine} makes three facts immediate.

First, if \(p_S\) is uniform, then \(q_S\) is uniform.

Second, if \(P=1/V\) (chance performance), then \(q_S\) is uniform for
any \(p_S\).

Third, if \(P\neq 1/V\), then \(q_S\) is non-uniform if and only if
\(p_S\) is non-uniform. In particular, for informative decoders with
\(P>1/V\), any non-uniformity in the in-vocabulary language distribution
propagates to the output distribution. Consequently,
\[
H(q_S)<\log_2 V
\]
whenever \(p_S\) is non-uniform and \(P\neq 1/V\), since the uniform
distribution uniquely maximises entropy on a finite alphabet.

\section{Closed Form OVMI Pseudocode}
\label{app:ovmi-pseudocode}

In Figure~\ref{alg:ovmi}, we provide Python pseudocode for computing OVMI under the homogeneous (single accuracy value) and symmetric error assumption. Note that when computing OVMI, we typically exclude words with fewer than five instances to ensure accuracy estimates for $P_\mathrm{macro}$ are reliable.

\begin{figure}[t]             
  \begin{lstlisting}[language=Python,basicstyle=\ttfamily\small]
import numpy as np
from scipy.special import xlogy

_LN2 = np.log(2.0)

def _xlog2y(x, y):
    """x * log2(y), with the convention 0 * log2(0) = 0."""
    return xlogy(x, y) / _LN2

def compute_ovmi(subset, P_c, freqs, p_freq):
    """
    subset: indices of words in the candidate vocabulary
    P_c:    macro-averaged classification accuracy over the subset
    freqs:  reference unigram frequencies (aligned with subset)
    p_freq: total number of words in the reference distribution
    Returns the open-vocabulary mutual information (OVMI).
    """
    V = len(subset)
    if V <= 1:
        return 0.0

    # (1) Coverage: natural-language mass spanned by the subset
    coverage = freqs.sum() / p_freq

    # (2) Input distribution p_s(x) over subset words
    if freqs.sum() != 0:
        p_x = freqs / freqs.sum()
    else:
        return 0

    # (3) Output distribution q(y) assuming errors spread uniformly
    #     over the V-1 incorrect classes
    p_err = (1 - P_c) / (V - 1)
    q_y = p_err + p_x * (P_c - p_err)

    # (4) Marginal entropy of predictions
    H_Y = -np.sum(_xlog2y(q_y, q_y))

    # (5) Conditional entropy
    H_YX = -_xlog2y(P_c, P_c) - _xlog2y(1 - P_c, p_err)

    # (6) Mutual information, floored at 0 to absorb numerical noise
    MI = max(0.0, H_Y - H_YX)

    # (7) Weight by natural-language coverage
    return coverage * MI
  \end{lstlisting}

  \caption{OVMI computation for a candidate vocabulary
  subset with a single accuracy scalar.}
  \label{alg:ovmi}
  \end{figure}

\section{Table of Results}
\label{app:table}

Table~\ref{tab:main} provides the full OVMI scores and uncertainties for the systems discussed in the main text across all the reference distributions explored in this work.

\begin{table*}[t]
\caption{\textbf{Cross-study comparison.} Each cell reports OVMI in bits, with OVMI/$H(p)$ in parentheses and uncertainty beneath. Rows are ordered within each block by SUBTLEX--UK OVMI. Superscripts on $P$ identify balanced top-1 accuracy ($\mathrm{b}$), isolated-word accuracy ($\mathrm{a}$), and $1-\mathrm{WER}$ ($\mathrm{w}$); $\dagger$ indicates that $P = 1 - \mathrm{WER}$ is a lower bound. Bracketed ranges are 95\% sampling intervals for invasive systems using Wilson score intervals for isolated-word accuracies and published sentence/trial intervals for WER-derived rows. $\pm$ values are OVMI uncertainties obtained by mapping mean $P$ plus or minus one SEM across three training seeds for non-invasive systems.}
\label{tab:main}
\centering
\scriptsize
\setlength{\tabcolsep}{1.1pt}
\renewcommand{\arraystretch}{1.12}
\begin{tabular*}{\textwidth}{@{\extracolsep{\fill}}lrr*{4}{r}@{}}
\toprule
System & $V$ & $P$ & \shortstack{Broad spoken\\SUBTLEX--UK} & \shortstack{Conversational\\Switchboard} & \shortstack{AAC / clinical\\UCV} & \shortstack{Narrative prose\\Sherlock} \\
 & & & \multicolumn{1}{c}{\scriptsize $H(p)=9.77$} & \multicolumn{1}{c}{\scriptsize $H(p)=8.27$} & \multicolumn{1}{c}{\scriptsize $H(p)=4.30$} & \multicolumn{1}{c}{\scriptsize $H(p)=8.44$} \\
\midrule
\multicolumn{7}{l}{\textbf{Attempted speech (invasive)}} \\
Card (+LM) & 125k & 97.5\%$^{\mathrm{w}\dagger}$ & \shortstack[r]{$9.156$ {\tiny (93.7\%)}\\$[9.093,9.208]$} & \shortstack[r]{$8.015$ {\tiny (96.9\%)}\\$[7.964,8.057]$} & \shortstack[r]{$4.197$ {\tiny (97.5\%)}\\$[4.171,4.218]$} & \shortstack[r]{$8.096$ {\tiny (95.9\%)}\\$[8.046,8.138]$} \\
\addlinespace[1.2pt]
Willett (+LM) & 125k & 76.2\%$^{\mathrm{w}\dagger}$ & \shortstack[r]{$7.038$ {\tiny (72.0\%)}\\$[6.835,7.232]$} & \shortstack[r]{$6.227$ {\tiny (75.3\%)}\\$[6.052,6.394]$} & \shortstack[r]{$3.279$ {\tiny (76.2\%)}\\$[3.189,3.365]$} & \shortstack[r]{$6.311$ {\tiny (74.7\%)}\\$[6.135,6.479]$} \\
\addlinespace[1.2pt]
Willett (isolated) & 50 & 94.0\%$^{\mathrm{a}}$ & \shortstack[r]{$0.656$ {\tiny (6.7\%)}\\$[0.637,0.672]$} & \shortstack[r]{$0.949$ {\tiny (11.5\%)}\\$[0.922,0.971]$} & \shortstack[r]{$1.822$ {\tiny (42.3\%)}\\$[1.776,1.859]$} & \shortstack[r]{$0.239$ {\tiny (2.8\%)}\\$[0.232,0.244]$} \\
\addlinespace[1.2pt]
Willett (+LM) & 50 & 90.9\%$^{\mathrm{w}\dagger}$ & \shortstack[r]{$0.621$ {\tiny (6.4\%)}\\$[0.599,0.642]$} & \shortstack[r]{$0.900$ {\tiny (10.9\%)}\\$[0.868,0.929]$} & \shortstack[r]{$1.737$ {\tiny (40.4\%)}\\$[1.681,1.789]$} & \shortstack[r]{$0.227$ {\tiny (2.7\%)}\\$[0.219,0.234]$} \\
\addlinespace[1.2pt]
Moses (+LM) & 50 & 74.4\%$^{\mathrm{w}\dagger}$ & \shortstack[r]{$0.459$ {\tiny (4.7\%)}\\$[0.360,0.539]$} & \shortstack[r]{$0.670$ {\tiny (8.1\%)}\\$[0.527,0.783]$} & \shortstack[r]{$1.321$ {\tiny (30.7\%)}\\$[1.053,1.529]$} & \shortstack[r]{$0.169$ {\tiny (2.0\%)}\\$[0.133,0.197]$} \\
\addlinespace[1.2pt]
Moses (isolated) & 50 & 47.1\%$^{\mathrm{a}}$ & \shortstack[r]{$0.238$ {\tiny (2.4\%)}\\$[0.230,0.245]$} & \shortstack[r]{$0.350$ {\tiny (4.2\%)}\\$[0.339,0.361]$} & \shortstack[r]{$0.711$ {\tiny (16.5\%)}\\$[0.690,0.733]$} & \shortstack[r]{$0.088$ {\tiny (1.0\%)}\\$[0.085,0.091]$} \\
\midrule
\multicolumn{7}{l}{\textbf{Perceived speech (non-invasive)}} \\
Tang 2023 & 6867 & 6.7\%$^{\mathrm{w}\dagger}$ & \shortstack[r]{$0.357$ {\tiny (3.6\%)}\\$\pm0.030$} & \shortstack[r]{$0.376$ {\tiny (4.5\%)}\\$\pm 0.031$} & \shortstack[r]{$0.266$ {\tiny (6.2\%)}\\$\pm 0.020$} & {\shortstack[r]{$0.348$ {\tiny (4.1\%)}\\$\pm 0.029$}} \\
\addlinespace[1.2pt]
LibriBrain100 2025 & 50 & 25.8\%$^{\mathrm{b}}$ & \shortstack[r]{$0.237$ {\tiny (2.4\%)}\\$\pm 0.003$} & \shortstack[r]{$0.275$ {\tiny (3.3\%)}\\$\pm 0.003$} & \shortstack[r]{$0.316$ {\tiny (7.3\%)}\\$\pm 0.004$} & \shortstack[r]{$0.221$ {\tiny (2.6\%)}\\$\pm 0.003$} \\
\addlinespace[1.2pt]
Armeni 2022 & 50 & 20.8\%$^{\mathrm{b}}$ & \shortstack[r]{$0.176$ {\tiny (1.8\%)}\\$\pm 0.004$} & \shortstack[r]{$0.192$ {\tiny (2.3\%)}\\$\pm 0.004$} & \shortstack[r]{$0.254$ {\tiny (5.9\%)}\\$\pm 0.005$} & \shortstack[r]{$0.189$ {\tiny (2.2\%)}\\$\pm 0.004$} \\
\addlinespace[1.2pt]
MEG-MASC 2023 & 50 & 8.5\%$^{\mathrm{b}}$ & \shortstack[r]{$0.033$ {\tiny (0.3\%)}\\$\pm 0.003$} & \shortstack[r]{$0.038$ {\tiny (0.5\%)}\\$\pm 0.004$} & \shortstack[r]{$0.059$ {\tiny (1.4\%)}\\$\pm 0.006$} & \shortstack[r]{$0.035$ {\tiny (0.4\%)}\\$\pm 0.003$} \\
\bottomrule
\end{tabular*}
\end{table*}

\section{Cross-Study OVMI Estimation Details}
\label{app:estimation}

This appendix describes how we estimate OVMI from the summary statistics
reported by existing speech-decoding studies and how we construct the
reference distributions used in Section~\ref{sec:experiments}. The general
definition of OVMI in Proposition~\ref{prop:ovmi-general} requires the
decoder channel $K_S(y\mid x)$. In retrospective comparisons, however,
published studies rarely report complete confusion matrices and instead
typically provide a vocabulary size together with accuracy or word error
rate (WER). We therefore use the homogeneous symmetric-channel estimator
from Eq.~\eqref{eq:ovmi-homogeneous} whenever a more detailed channel
estimate is unavailable.

\subsection{Estimation from a Scalar Accuracy}
\label{app:scalar-estimation}

For a decoder vocabulary $S$ of size $V$, and given a single correct-decoding probability $P$, we approximate the
in-vocabulary channel as
\[
K_S(y\mid x)=
\begin{cases}
P, & y=x,\\[2mm]
\dfrac{1-P}{V-1}, & y\neq x.
\end{cases}
\]
The induced output distribution is therefore
\[
q_S(y)
=
Pp_S(y)
+
\frac{1-P}{V-1}\left[1-p_S(y)\right],
\]
and the conditional entropy of the symmetric channel is
\[
h_V(P)
=
-P\log_2 P
-(1-P)\log_2\frac{1-P}{V-1}.
\]
The retrospective estimate used throughout the cross-study comparison is
then
\[
\widehat I_{\mathrm{OVMI}}
=
C(S)\left[H(q_S)-h_V(P)\right].
\]
This approximation assumes only that all supported words have the same
correct-decoding probability and that errors are distributed uniformly
among the remaining $V-1$ outputs. The reference distribution itself
remains non-uniform through $C(S)$, $p_S$, and $q_S$.

Where available, we use macro rather than micro accuracy for $P$. In
particular,
\[
P_{\mathrm{macro}}
=
\frac{1}{V}\sum_{x\in S}
    \Pr(\hat Y=x\mid X=x),
\]
which summarises the reliability of an average supported word without
additionally weighting common words by their frequency. Frequency is
already represented by the external distribution $p$.

\subsection{Conversion of Reported Performance Measures}
\label{app:reported-metrics}

\paragraph{Balanced top-1 accuracy.}
For the non-invasive word-classification results, the reported balanced
top-1 accuracy is the macro-average of the class-wise recalls.

\paragraph{Isolated-word accuracy.}
For studies reporting accuracy on an isolated $V$-way word-classification
task, we use the reported mean correct-classification probability directly
as $P$. This gives $P=0.471$ for the isolated 50-word experiment of
\citet{Moses2021NeuroprosthesisFD} and $P=0.940$ for the corresponding
50-word result of \citet{Willett2023AHS}.

\paragraph{Word error rate.}
Continuous speech decoders generally report WER rather than word
classification accuracy. For a reference containing $N$ words, with
$N_{\mathrm{sub}}$ substitutions, $N_{\mathrm{del}}$ deletions, and
$N_{\mathrm{ins}}$ insertions,
\[
\mathrm{WER}
=
\frac{
N_{\mathrm{sub}}+N_{\mathrm{del}}+N_{\mathrm{ins}}
}{N}.
\]
The proportion of intended words that are neither substituted nor deleted
is
\[
P_{\mathrm{correct}}
=
1-\frac{N_{\mathrm{sub}}+N_{\mathrm{del}}}{N}
=
1-\mathrm{WER}
+\frac{N_{\mathrm{ins}}}{N}.
\]
Hence
\[
P_{\mathrm{correct}}\geq 1-\mathrm{WER}.
\]
When only WER is available, we therefore set
\[
P=1-\mathrm{WER}.
\]
This is a conservative estimate of per-intended-word correctness because
insertions increase WER without corresponding to a failure to recover an
intended word. For the language-model-assisted results, this gives
$P=0.744$ for Moses, $P=0.909$ and $0.762$ for the 50-word and 125k-word
Willett systems, respectively, and $P=0.975$ for Card.

\subsection{Reference Distributions}
\label{app:reference-details}

All reference distributions are unigram distributions over word types.
For a corpus-derived reference, we estimate
\[
p(w)=\frac{n(w)}{\sum_{u\in\Omega}n(u)},
\]
where $n(w)$ is the number of occurrences of word $w$ in the reference
corpus after lower-casing all words. For a decoder vocabulary $S$, words
outside $S$ remain part of the reference distribution and contribute to
the uncovered probability mass $1-C(S)$; $p$ is therefore never
renormalised to the decoder vocabulary before computing coverage.

\paragraph{Broad spoken English: SUBTLEX-UK.}
Our default reference is SUBTLEX-UK
\citep{vanHeuven2014SubtlexUKAN}, a word-frequency norm derived from
British film and television subtitles. We normalise the SUBTLEX-UK
frequency counts to obtain $p$. This distribution has entropy
$H(p)=9.77$ bits after the preprocessing used in our evaluation and is
intended as a broad spoken-English reference.

\paragraph{Conversation: Switchboard.}
For conversational speech, we construct an empirical unigram distribution
from the 36-call Switchboard sample distributed with NLTK. Word counts are
aggregated across the sample and normalised to sum to one. The resulting
reference has entropy $H(p)=8.27$ bits.

\paragraph{AAC: Universal Core Vocabulary.}
For an assistive communication reference, we use the 36-word Universal
Core Vocabulary (UCV) of \citet{erickson2019universal}. UCV specifies a
set of words rather than a probability distribution. We therefore assign
each UCV word its corresponding SUBTLEX-UK frequency and renormalise
within the UCV set,
\[
p_{\mathrm{UCV}}(w)
=
\frac{f_{\mathrm{SUBTLEX}}(w)}
{\sum_{u\in\mathrm{UCV}}f_{\mathrm{SUBTLEX}}(u)},
\qquad w\in\mathrm{UCV}.
\]
This preserves natural differences in word frequency rather than treating
the 36 words as equiprobable. The resulting distribution has
$H(p)=4.30$ bits.

\looseness=-1 \paragraph{Narrative speech: Sherlock.}
For narrative prose, we use the held-out LibriBrain100 text corresponding to
Session~12, Chapter~1 of \textit{The Adventures of Sherlock Holmes} \citep{doyle1892sherlock}.
This session belongs to the LibriBrain100 test set and is not used to fit the
non-invasive decoders. We count the words occurring in the chapter and
normalise their empirical frequencies to form the narrative reference
distribution. The resulting distribution has entropy $H(p)=8.44$ bits.
This reference is intentionally narrower than SUBTLEX-UK and represents
the narrative-speech domain on which several of the non-invasive systems
are evaluated.

\subsection{Normalisation and Uncertainty}
\label{app:normalisation}

OVMI is reported in bits per intended word drawn from $p$. When comparing
performance across different reference distributions, we additionally
normalise by the entropy of that distribution,
\[
\mathrm{OVMI}_{\%}
=
100\,
\frac{I_{\mathrm{OVMI}}}{H(p)}.
\]
This quantity is the percentage of lexical information available under
the stated reference distribution that is conveyed by the system.

For invasive isolated-word results, uncertainty in the reported accuracy
is propagated through the OVMI estimator using Wilson score intervals.
For WER-derived results, we propagate the corresponding published
sentence- or trial-level uncertainty through $P=1-\mathrm{WER}$. For the
non-invasive experiments, we compute OVMI at the mean balanced accuracy
and at one standard error above and below that mean across training seeds.

\subsection{Large-Vocabulary Invasive Systems}
The 125k-word decoder lexicons used by \citet{Willett2023AHS} and
\citet{Card2024AnAA} come from a lexicon based on the CMU Pronouncing Dictionary, but are not published in full as far as we are aware. For coverage calculations, we use the CMU Pronouncing Dictionary ourselves, after collapsing pronunciation variants to unique orthographic
word forms. This yields 126,052 unique spellings. We retain the reported vocabulary size $V=125{,}000$ when evaluating the decoder channel.

\section{Vocabulary Optimisation Experiment Details}
\label{app:optimise}

\looseness=-1 \paragraph{Data and splits.}
We use subject 0 from LibriBrain100 and report matched-domain results for TIMIT,
Podcasts, and Sherlock. Reference word distributions are estimated exclusively
from the training portion of the corresponding domain. For Podcasts, sessions
1--28 form the training set, session 29 is validation, and session 30 is test.
For Sherlock, the training set comprises the training runs, with
\texttt{Sherlock1} session 11 used for validation and session 12 for test.
TIMIT is split at the utterance level using the canonical LibriBrain100 policy:
the 50 development speakers are used for validation and the 24 core-test
speakers for testing. The repeated SA utterances are excluded from validation
and test, as are training speakers sharing an SX sentence with a core-test
speaker. MOCHA-TIMIT training sessions A and D contribute only to decoder
training and candidate-word availability.

\paragraph{Candidate vocabulary and neural decoder.}
The candidate pool is constructed from word tokens with usable neural training
windows after removal of incomplete examples. Words must occur at
least five times in the pooled neural training data. Eligible words are ordered
by decreasing training count and the pool is capped
at 250 words. The same candidate pool is used for all domains, selection
methods, vocabulary sizes, and model seeds.

We train five independently initialised instances of the contrastive
brain-to-T5-embedding decoder by \citet{dAscoli2025TowardsDI} for 50 epochs with a batch size of 128. For each seed, the
checkpoint is selected using pooled validation contrastive top-1 accuracy. The model has no
vocabulary-specific classification head. After training, we construct one
normalised target embedding for each candidate word and cache the cosine score
between every validation or test neural prediction and all 250 candidates.
Consequently, evaluating a vocabulary requires only restricting the cached
score matrix to the corresponding columns and applying an argmax; the neural
model is never retrained for a particular method or vocabulary size.

\paragraph{Reference distribution.}
For domain \(d\), let \(p_d(w)\) denote the empirical probability of word \(w\)
in that domain's training text. Probabilities are defined over the full
training-text distribution rather than being renormalised over the 250-word
candidate pool. Thus, for a selected vocabulary \(\mathcal{V}\), its lexical
coverage is
\[
    C_d(\mathcal{V})
    = \sum_{w\in\mathcal{V}} p_d(w).
\]
No validation or test text contributes to \(p_d\).

To compute OVMI, a full confusion-channel estimate is used when every word in the trial
vocabulary has at least five validation examples. Otherwise, selection
falls back to the per-word symmetric-error OVMI approximation. At \(V=150\), this fallback was used for all five TIMIT
seeds and for Sherlock seeds 0, 3, and 5; all other final domain--size--seed
combinations used the full validation channel.

\paragraph{Statistical testing.}
For every domain and vocabulary size, OVMI and Frequency are
compared using a one-sided paired randomisation test.
We use \(B=100{,}000\) Monte Carlo permutations and calculate the one-sided
\(p\)-value with the finite-simulation correction
\[
    p_{\mathrm{raw}}
    =
    \frac{
        1+\sum_{b=1}^{B}\mathbf{1}[T_b\geq T_{\mathrm{obs}}]
    }{B+1}.
\]

Each \(V\) is tested separately. Holm's step-down correction is then applied
across the nine vocabulary sizes within each domain, controlling the error rate separately for TIMIT, Podcasts, and Sherlock.

\makeatletter
\if@preprint
\else
\section{Broader Impacts}
\label{app:impacts}

\looseness=-1 Restoring speech to paralysed patients could improve their quality of life. At present, clinical deployment remains distant and substantial work in signal quality, decoding accuracy, and user-centred design ought to precede any deployment. The framework introduced here is one step within that longer trajectory. We also note that neural decoding technologies raise privacy concerns, as they involve inferring linguistic content from brain activity. The present work uses only publicly available research datasets collected under their own ethics approvals. As decoding capabilities improve, the field will need norms around informed consent, data ownership, and the boundary between assistive and surveillant applications.
\fi
\makeatother

\makeatletter
\if@preprint\else
  \clearpage
  \input{checklist.tex}
\fi
\makeatother

\end{document}

%% file: sections/intro_new.tex
\section{Introduction}
\label{sec:introduction}

\looseness=-1 In the latter half of the 1980s, DARPA proposed a speech recognition competition in which different systems transcribed the same held-out recordings and were ranked by a single score, settling the long-standing dispute between statistical methods and hand-written rules~\citep{Pallett2003ALA,Donoho201750YO,koch2024protoscience}. Standardised benchmarks built on this principle, with a shared task, dataset, and metric, have since underpinned progress across much of the machine learning field \citep{deng2009, Mnih2015HumanlevelCT, wang2019}. In the same vein, the field of speech BCIs has begun to adopt the same approach, with new benchmarks enabling controlled comparisons of methods that decode neural activity into speech~\citep{willett2024brain, brain-to-text-25, landau2025pnpl, mantegna2026pnpl}.
However, across the wider field, studies cannot all be evaluated with a common benchmark because methods are often specific to different experimental settings. Moreover, speech BCIs typically support different vocabularies of words, so studies report scores that are conditional on the supported vocabulary. As a result, scores from different studies are hard to compare.

If the vocabulary supported by a speech BCI is sufficiently large, then it can express most plausible sentences and the choice of supported words matters little. 
In practice, however, smaller constrained vocabularies are more typical of current systems~\citep{Moses2021NeuroprosthesisFD, Willett2023AHS, ozdogan2025libribrain, dAscoli2025TowardsDI, jayalath2026meg}, where studies construct vocabularies in one of two ways. Some select the most frequent $V$ words
from the available data \citep{dAscoli2025TowardsDI, jayalath2026meg}; others curate a selection of words for a particular use case~\citep{Moses2021NeuroprosthesisFD}.
Consequently, two systems with the same vocabulary size and accuracy may support very different fractions of what a user wishes to communicate. Conversely, imposing the same vocabulary across datasets would create an unnatural and unfair comparison as those words may occur with very different frequencies, or not at all. Thus, neither standardising a universal vocabulary nor using existing study-specific vocabularies provide a satisfactory way to compare studies.

\looseness=-1 More fundamentally, no benchmark is able to capture the full range of speech BCI studies. Systems differ in recording method, from intracortical implants to non-invasive MEG and EEG; in speech task, e.g. attempted or perceived speech; in participant population, e.g. healthy volunteers or paralysed patients; and experimental protocol. These differences often make evaluation on a shared neural dataset impossible. Thus, benchmarks provide controlled comparisons of methods within an experimental setting but do not by themselves tell us how capabilities measured in different settings relate to one another. For example, an intracortical attempted-speech decoder and a non-invasive perceived-speech decoder cannot be compared as if they were evaluated under the same conditions, yet we may still wish to ask how much communicative capability each demonstrates relative to the same communication objective. This question is also useful retrospectively, because many studies fall outside of the scope of a benchmark, and prospectively, since benchmarks may eventually saturate or be superseded. Therefore, a measure of communication should be meaningful across studies and successive benchmarks without depending on any one benchmark staying relevant.

\looseness=-1 
A benchmark also evaluates speech decoding on a particular language distribution. For example, the 2025 PNPL competition used Sherlock Holmes stories \citep{landau2025pnpl}. Repeated optimisation against such a benchmark may favour methods that decode only this language distribution well. We may instead want to know what a method's decoding capability measured on a benchmark means relative to other communication distributions, such as conversational or caregiving speech. Collecting a new neural dataset for every such communication objective would be impractical. We therefore need to separate the language distribution used to measure decoding capability from the language distribution against which communication is assessed. Resolving these issues requires answering two underlying questions:
\vspace{0.5\baselineskip}
\begin{thesisbox}
\textbf{Unanswered questions for measuring progress in speech BCIs}\\[4pt]
1. \underline{What distribution} (over words) should a speech BCI enable a user to communicate?\\
2. \underline{How much information} from that distribution can a speech BCI convey?
\end{thesisbox}
\vspace{0.5\baselineskip}
\begin{figure}[t]
    \centering
    \includegraphics[width=1.0\linewidth]{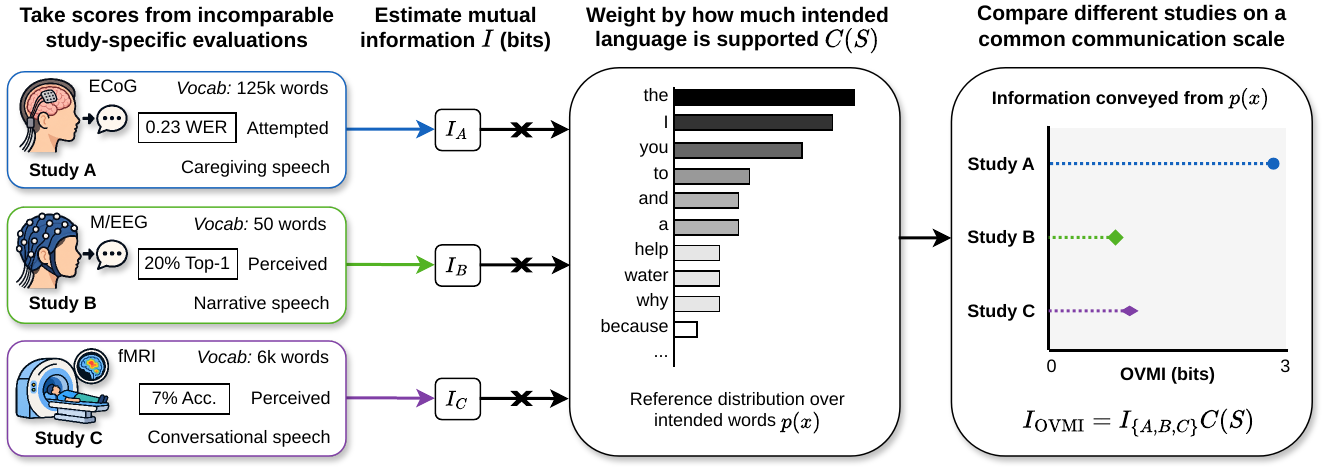}
    \caption{\looseness=-1 \textbf{OVMI maps incomparable evaluations onto a common scale.} Speech BCI studies differ in experimental setting, supported vocabulary, and test data, so their reported scores are not comparable. Each study provides a decoding score over its supported vocabulary $S$, from which we estimate the information conveyed within that vocabulary, $I$. We define a common reference distribution $p(x)$ over the words a user may wish to communicate, and let $C(S)$ denote the probability that a word from $p(x)$ is supported by the system. OVMI weights the decoded information $I$ by $C(S)$, yielding a score that measures performance against $p(x)$. Evaluating all systems against the same $p(x)$ therefore allows comparing heterogeneous speech BCI studies on a common communication scale.}
    \label{fig:concept}
    \vspace{-1em}
\end{figure}

\looseness=-1 \textbf{(1) Defining the communication distribution.}\quad
A speech BCI should ideally allow a user to convey what they would otherwise say aloud. To illustrate the problem, a decoder that is perfectly accurate on a fifty-word vocabulary succeeds only insofar as those fifty words capture what the user wants to say, and a vocabulary suitable for caregiving may be poorly suited to conversation. Moreover, the words a user intends are not uniformly distributed. We therefore model the distribution of words a user wants to communicate as a \defterm{reference distribution~$p$}. This makes the target communication domain explicit and provides a common distribution against which any decoder can be assessed. Thus, the amount of information a decoder conveys depends on both the decoder and the distribution of intended messages.

\looseness=-1 \textbf{(2) Measuring communication against a reference distribution.}\quad
Having specified what a user may wish to communicate through $p$, we ask how much of that communication a decoder can convey.
While accuracy and word error rate (WER) measure decoding fidelity, they remain conditional on the vocabulary supported by the decoder. Moreover, an error rate alone does not quantify how much information has been conveyed as, for example, the same classification accuracy resolves very different amounts of uncertainty when choosing among a vocabulary of ten words versus one thousand. Mutual information provides a natural quantity by measuring how much observing the decoder output reduces uncertainty about the intended message.

In BCIs, the standard information-theoretic measure is Wolpaw's information transfer rate and its refinements~\citep{Wolpaw1998EEGbasedCI,Wolpaw2002BraincomputerIF,Speier2013EvaluatingTB}. These measures quantify information within a predefined set of symbols. This is natural for classical closed-symbol interfaces, such as character spellers~\citep{Farwell1988TalkingOT}, but not when a user may intend words outside of a decoder's vocabulary. 
We therefore derive \underline{open-vocabulary mutual information} (\defterm{OVMI}), which evaluates lexical information when the intended word is drawn from $p$.
As Section~\ref{sec:ovmi-theory} shows,
OVMI emerges naturally from the decomposition of mutual information where OVMI is in-vocabulary mutual information weighted by the \defterm{lexical coverage} of $p$, the probability that an intended word is supported.
Evaluating systems against the same \(p\) therefore
allows heterogeneous studies to be compared on a common scale (Figure~\ref{fig:concept}).

\looseness=-1
\textbf{Contributions.}\quad
We (i) show why
conventional measures can overstate communication and derive OVMI, an information-theoretic measure of lexical information
transfer relative to an explicit reference communication distribution
(Sections~\ref{sec:ovmi-theory}--\ref{sec:overestimate}); (ii) use OVMI to
compare heterogeneous speech BCI systems on a common scale, revealing how much systems depend on lexical coverage, decoding fidelity, and
the intended communication domain (Sections~\ref{sec:common-scale}--\ref{sec:refdist});
and (iii) show that OVMI can also guide vocabulary selection, improving
speech BCI accuracy across three speech domains (Section~\ref{sec:optimise}).

%% file: checklist.tex
\section*{NeurIPS Paper Checklist}

\begin{enumerate}

\item {\bf Claims}
    \item[] Question: Do the main claims made in the abstract and introduction accurately reflect the paper's contributions and scope?
    \item[] Answer: \answerYes{} %
    \item[] Justification: We confirm that the stated contributions reflect the paper's scope as well as its theoretical and empirical contributions.

\item {\bf Limitations}
    \item[] Question: Does the paper discuss the limitations of the work performed by the authors?
    \item[] Answer: \answerYes{} %
    \item[] Justification: We explicitly discuss limitations at length in Section~\ref{sec:discussion}.

\item {\bf Theory assumptions and proofs}
    \item[] Question: For each theoretical result, does the paper provide the full set of assumptions and a complete (and correct) proof?
    \item[] Answer: \answerYes{} %
    \item[] Justification: We state all assumptions in Section~\ref{sec:ovmi-theory} and in Appendix~\ref{app:ovmi-proofs}. We provide proofs of all of our propositions and corollaries as well as further helpful derivations in Appendix~\ref{app:ovmi-proofs}.

    \item {\bf Experimental result reproducibility}
    \item[] Question: Does the paper fully disclose all the information needed to reproduce the main experimental results of the paper to the extent that it affects the main claims and/or conclusions of the paper (regardless of whether the code and data are provided or not)?
    \item[] Answer: \answerYes{} %
    \item[] Justification: We describe our experimental setup in Section~\ref{sec:experiments} and provide simple Python pseudocode for calculating OVMI in Appendix~\ref{app:ovmi-pseudocode}.

\item {\bf Open access to data and code}
    \item[] Question: Does the paper provide open access to the data and code, with sufficient instructions to faithfully reproduce the main experimental results, as described in supplemental material?
    \item[] Answer: \answerYes{} %
    \item[] Justification: We provide code and instructions for reproduction of our experiments in the supplementary materials. We will release this publicly after the review process.

\item {\bf Experimental setting/details}
    \item[] Question: Does the paper specify all the training and test details (e.g., data splits, hyperparameters, how they were chosen, type of optimizer) necessary to understand the results?
    \item[] Answer: \answerYes{} %
    \item[] Justification: We provide all details on our data splits and hyperparameters in Appendix~\ref{app:dsplits-hparams}.

\item {\bf Experiment statistical significance}
    \item[] Question: Does the paper report error bars suitably and correctly defined or other appropriate information about the statistical significance of the experiments?
    \item[] Answer: \answerYes{} %
    \item[] Justification: We show standard error over five seeds with error bars or shading on all plots.

\item {\bf Experiments compute resources}
    \item[] Question: For each experiment, does the paper provide sufficient information on the computer resources (type of compute workers, memory, time of execution) needed to reproduce the experiments?
    \item[] Answer: \answerYes{} %
    \item[] Justification: We describe these details at the end of Appendix~\ref{app:dsplits-hparams}.
    
\item {\bf Code of ethics}
    \item[] Question: Does the research conducted in the paper conform, in every respect, with the NeurIPS Code of Ethics \url{https://neurips.cc/public/EthicsGuidelines}?
    \item[] Answer: \answerYes{} %
    \item[] Justification: This research follows and complies with the NeurIPS Code of Ethics.

\item {\bf Broader impacts}
    \item[] Question: Does the paper discuss both potential positive societal impacts and negative societal impacts of the work performed?
    \item[] Answer: \answerYes{} %
    \item[] Justification: We discuss broader impacts of our research in Appendix~\ref{app:impacts}.
    
\item {\bf Safeguards}
    \item[] Question: Does the paper describe safeguards that have been put in place for responsible release of data or models that have a high risk for misuse (e.g., pre-trained language models, image generators, or scraped datasets)?
    \item[] Answer: \answerNA{} %
    \item[] Justification: We do not release any new models or data.

\item {\bf Licenses for existing assets}
    \item[] Question: Are the creators or original owners of assets (e.g., code, data, models), used in the paper, properly credited and are the license and terms of use explicitly mentioned and properly respected?
    \item[] Answer: \answerYes{} %
    \item[] Justification: We cite all papers that introduce code, models, or datasets used in this work. We include the license terms in the reference for datasets.

\item {\bf New assets}
    \item[] Question: Are new assets introduced in the paper well documented and is the documentation provided alongside the assets?
    \item[] Answer: \answerYes{} %
    \item[] Justification: We provide a README file with instructions to reproduce our experiments in the supplementary materials.

\item {\bf Crowdsourcing and research with human subjects}
    \item[] Question: For crowdsourcing experiments and research with human subjects, does the paper include the full text of instructions given to participants and screenshots, if applicable, as well as details about compensation (if any)? 
    \item[] Answer: \answerNA{} %
    \item[] Justification: We do not collect any new data. All experiments use data from cited publicly available datasets with their own ethics approvals.

\item {\bf Institutional review board (IRB) approvals or equivalent for research with human subjects}
    \item[] Question: Does the paper describe potential risks incurred by study participants, whether such risks were disclosed to the subjects, and whether Institutional Review Board (IRB) approvals (or an equivalent approval/review based on the requirements of your country or institution) were obtained?
    \item[] Answer: \answerNA{} %
    \item[] Justification: We do not collect any new data. All experiments use data from cited publicly available datasets with their own ethics approvals.

\item {\bf Declaration of LLM usage}
    \item[] Question: Does the paper describe the usage of LLMs if it is an important, original, or non-standard component of the core methods in this research? Note that if the LLM is used only for writing, editing, or formatting purposes and does \emph{not} impact the core methodology, scientific rigor, or originality of the research, declaration is not required.
    \item[] Answer: \answerYes{} %
    \item[] Justification: We declare non-standard use of LLMs for assisting in sketching proofs during the development of this work. We describe this at the end of Appendix~\ref{app:dsplits-hparams}. We independently verified these proof outlines by hand for correctness and rigour, then developed them into the full proofs in Appendix~\ref{app:ovmi-proofs}.

\end{enumerate}